\documentclass{article}

\usepackage[final, nonatbib]{nips_2017}
\usepackage{amsmath, amsfonts, amsthm, bm, tikz, xspace, wrapfig, subcaption, tabularx, booktabs, nicefrac}
\tikzset{
  >=latex,
  every node/.style={thick, circle, draw, minimum height=1.7em, inner sep=0pt, text centered},
  every path/.style={thick},
}

\usetikzlibrary{calc, shapes.geometric, fadings, decorations.pathreplacing}

\definecolor{Crimson}{RGB}{165, 28, 48}
\definecolor{SteelBlue}{RGB}{70, 130, 180}
\definecolor{DarkRed}{RGB}{165, 0, 38}
\definecolor{LightRed}{RGB}{215, 48, 39}
\definecolor{DarkOrange}{RGB}{244, 109, 67}
\definecolor{LightOrange}{RGB}{253, 174, 97}
\definecolor{MidYellow}{RGB}{254, 224, 144}
\definecolor{MiddleBlue}{RGB}{116, 173, 209}
\definecolor{SkyBlue}{RGB}{69, 117, 180}
\definecolor{SeaBlue}{RGB}{49, 54, 149}
\definecolor{Pink}{HTML}{E91E63}
\definecolor{DeepOrange}{HTML}{FF5722}
\definecolor{Blue}{HTML}{2196F3}
\definecolor{Orange}{HTML}{FF9800}
\definecolor{Cyan}{HTML}{00BCD4}
\definecolor{Green}{HTML}{4CAF50}
\definecolor{DeepPurple}{HTML}{673AB7}
\definecolor{Purple}{HTML}{9C27B0}
\definecolor{Red}{HTML}{F44336}
\definecolor{Brown}{HTML}{795548}
\definecolor{Teal}{HTML}{009688}
\definecolor{Yellow}{HTML}{FFEB3B}
\definecolor{Grey}{HTML}{9E9E9E}
\definecolor{Lime}{HTML}{CDDC39}
\definecolor{Indigo}{HTML}{3F51B5}
\definecolor{LightGreen}{HTML}{8BC34A}
\definecolor{BlueGrey}{HTML}{607D8B}
\definecolor{LightBlue}{HTML}{03A9F4}
\definecolor{Amber}{HTML}{FFC107}

\newcommand{\bR}{\mathbb{R}}

\newcommand{\cN}{\mathcal{N}}
\newcommand{\cG}{\mathcal{G}}
\newcommand{\tcG}{\tilde{\mathcal{G}}}
\newcommand{\bP}{\mathbb{P}}
\newcommand{\hY}{R}

\newcommand{\hR}{R}

\def\eqp{\: .}
\def\eqc{\: ,}

\newcommand*{\ie}{i.e.\@\xspace}
\newcommand*{\wrt}{w.r.t.\@\xspace}

\newcommand{\secref}[1]{Section~\ref{#1}}

\newcommand{\figref}[1]{Figure~\ref{#1}}
\newcommand{\Figref}[1]{Figure~\ref{#1}}

\newcommand{\usub}[2]{\underset{#1}{\underbrace{#2}}}
\newcommand\given{{\,|\,}}

\DeclareMathOperator{\E}{\mathbb{E}}
\DeclareMathOperator{\indep}{\perp\!\!\!\perp}

\newtheorem{theorem}{Theorem}
\newtheorem*{theorem*}{Theorem}

\newtheorem*{lemma*}{Lemma}

\newtheorem{corollary}{Corollary}
\newtheorem*{corollary*}{Corollary}

\newtheorem{proposition}{Proposition}
\newtheorem*{proposition*}{Proposition}

\theoremstyle{definition}
\newtheorem{definition}{Definition}

\usepackage[%
  backend=bibtex,
  style=numeric-comp,
  hyperref=true,
  sorting=none,
  maxbibnames=99,
  url=false,
  isbn=false,
  hyperref
  ]{biblatex}
\usepackage{enumitem}

\setlist[enumerate]{leftmargin=*}

\newif\ifproof
\prooffalse

\author{
Niki Kilbertus\textsuperscript{$\dagger$}\textsuperscript{$\ddagger$}\\
\texttt{nkilbertus@tue.mpg.de}
\And
Mateo Rojas-Carulla\textsuperscript{$\dagger$}\textsuperscript{$\ddagger$}\\
\texttt{mrojas@tue.mpg.de}
\And
Giambattista Parascandolo\textsuperscript{$\dagger$}\textsuperscript{\S}\\
\texttt{gparascandolo@tue.mpg.de}
\AND
Moritz Hardt\textsuperscript{$\ast$}\\
\texttt{hardt@berkeley.edu}
\And
Dominik Janzing\textsuperscript{$\dagger$}\\
\texttt{janzing@tue.mpg.de}
\And
Bernhard Sch\"olkopf\textsuperscript{$\dagger$}\\
\texttt{bs@tue.mpg.de}
\AND
\normalfont
\normalsize\textsuperscript{$\dagger$}Max Planck Institute for Intelligent Systems\\
\normalsize\textsuperscript{$\ddagger$}University of Cambridge\\
\normalsize\textsuperscript{\S}Max Planck ETH Center for Learning Systems\\
\normalsize\textsuperscript{$\ast$}University of California, Berkeley
}

\title{Avoiding Discrimination through Causal Reasoning}

\begin{document}
\maketitle

\begin{abstract}
  Recent work on fairness in machine learning has focused on various statistical
  discrimination criteria and how they trade off. Most of these criteria are
  observational: They depend only on the joint distribution of predictor,
  protected attribute, features, and outcome. While convenient to work with,
  observational criteria have severe inherent limitations that prevent them from
  resolving matters of fairness conclusively.

  Going beyond observational criteria, we frame the problem of discrimination
  based on protected attributes in the language of causal reasoning. This
  viewpoint shifts attention from ``What is the right fairness criterion?'' to
  ``What do we want to assume about our model of the causal data generating process?''
  Through the lens of causality, we make several contributions. First, we
  crisply articulate why and when observational criteria fail, thus formalizing
  what was before a matter of opinion. Second, our approach exposes previously
  ignored subtleties and why they are fundamental to the problem. Finally, we
  put forward natural causal non-discrimination criteria and develop algorithms
  that satisfy them.
\end{abstract}

\section{Introduction}
\label{sec:intro}

As machine learning progresses rapidly, its societal impact has come under
scrutiny. An important concern is potential discrimination based on protected
attributes such as gender, race, or religion. Since learned predictors and risk
scores increasingly support or even replace human judgment, there is an
opportunity to formalize what harmful discrimination means and to design
algorithms that avoid it. However, researchers have found it difficult to agree
on a single measure of discrimination. As of now, there are several competing
approaches, representing different opinions and striking different trade-offs.
Most of the proposed fairness criteria are observational: They depend only on
the joint distribution of predictor~$\hY,$ protected attribute~$A$,
features~$X$, and outcome $Y.$ For example, the natural requirement that~$\hY$
and~$A$ must be statistically independent is referred to as \emph{demographic
parity}. Some approaches transform the features~$X$ to obfuscate the information
they contain about~$A$~\cite{Zemel2013}. The recently proposed \emph{equalized
odds} constraint~\cite{Hardt2016} demands that the predictor $\hY$ and the
attribute $A$ be independent conditional on the actual outcome $Y.$ All three
are examples of observational approaches.

A growing line of work points at the insufficiency of existing definitions.
Hardt, Price and Srebro~\cite{Hardt2016} construct two scenarios with
\emph{intuitively} different social interpretations that admit identical joint
distributions over $(\hY, A, Y, X)$. Thus, no observational criterion can
distinguish them. While there are non-observational criteria, notably the early
work on individual fairness~\cite{Dwork2012}, these have not yet gained
traction.  So, it might appear that the community has reached an impasse.

\subsection{Our contributions}

We assay the problem of discrimination in machine learning in the language of
causal reasoning. This viewpoint supports several contributions:

\begin{itemize}
  \item Revisiting the two scenarios proposed in~\cite{Hardt2016}, we articulate
    a natural causal criterion that formally distinguishes them. In particular,
    we show that observational criteria are unable to determine if a protected
    attribute has \emph{direct causal influence} on the predictor that is not
    mitigated by \emph{resolving} variables.
  \item We point out subtleties in fair decision making that arise naturally
    from a causal perspective, but have gone widely overlooked in the past.
    Specifically, we formally argue for the need to distinguish between the
    underlying concept behind a protected attribute, such as race or gender, and
    its \emph{proxies} available to the algorithm, such as visual features or
    name.
  \item We introduce and discuss two natural causal criteria centered around the
    notion of \emph{interventions} (relative to a causal graph) to formally
    describe specific forms of discrimination.
  \item Finally, we initiate the study of algorithms that avoid these forms of
    discrimination. Under certain linearity assumptions about the underlying
    causal model generating the data, an algorithm to remove a specific kind of
    discrimination leads to a simple and natural heuristic.
\end{itemize}

At a higher level, our work proposes a shift from trying to find a single
statistical fairness criterion to arguing about properties of the data and which
assumptions about the generating process are justified. Causality provides a
flexible framework for organizing such assumptions.

\subsection{Related work}

Demographic parity and its variants have been discussed in numerous papers,
e.g., \cite{Feldman2015, Zemel2013, Zafar2017, Edwards2015}. While demographic
parity is easy to work with, the authors of \cite{Dwork2012} already highlighted
its insufficiency as a fairness constraint. In an attempt to remedy the
shortcomings of demographic parity~\cite{Hardt2016} proposed two notions,
\emph{equal opportunity} and \emph{equal odds}, that were also considered
in~\cite{Zafar2017a}.  A review of various fairness criteria can be found
in~\cite{Berk2017}, where they are discussed in the context of criminal justice.
In~\cite{Kleinberg2016, Chouldechova2016} it has been shown that imperfect
predictors cannot simultaneously satisfy equal odds and \emph{calibration}
unless the groups have identical base rates, \ie{} rates of positive outcomes.

A starting point for our investigation is the unidentifiability result
of~\cite{Hardt2016}. It shows that observedvational criteria are too weak to
distinguish two intuitively very different scenarios. However, the work does not
provide a formal mechanism to articulate why and how these scenarios should be
considered different. Inspired by Pearl's causal interpretation of Simpson's
paradox~\cite[Section 6]{Pearl2009}, we propose causality as a way of coping
with this unidentifiability result.

An interesting non-observational fairness definition is the notion of
\emph{individual fairness}~\cite{Dwork2012} that assumes the existence of a
similarity measure on individuals, and requires that any two similar individuals
should receive a similar distribution over outcomes. More recent work lends
additional support to such a definition~\cite{Friedler2016}. From the
perspective of causality, the idea of a similarity measure is akin to the method
of \emph{matching} in counterfactual reasoning~\cite{Rosenbaum1983,
Qureshi2016}. That is, evaluating approximate counterfactuals by comparing
individuals with similar values of covariates excluding the protected attribute.

Recently,~\cite{Kusner2017} put forward one possible causal definition, namely
the notion of \emph{counterfactual fairness}. It requires modeling
counterfactuals on a per individual level, which is a delicate task. Even
determining the effect of \emph{race} at the group level is difficult; see the
discussion in~\cite{VanderWeele2014}. The goal of our paper is to assay a more
general causal framework for reasoning about discrimination in machine learning
without committing to a single fairness criterion, and without committing to
evaluating individual causal effects. In particular, we draw an explicit
distinction between the protected attribute (for which interventions are often
impossible in practice) and its proxies (which sometimes can be intervened
upon).

Moreover, causality has already been employed for the discovery of
discrimination in existing data sets by~\cite{Bonchi2015, Qureshi2016}. Causal
graphical conditions to identify \emph{meaningful partitions} have been proposed
for the discovery and prevention of certain types of discrimination by
preprocessing the data~\cite{Zhang2017b}. These conditions rely on the
evaluation of \emph{path specific effects}, which can be traced back all the way
to~\cite[Section 4.5.3]{Pearl2009}. The authors of~\cite{Nabi2017} recently
picked up this notion and generalized Pearl's approach by a constraint based
prevention of discriminatory path specific effects arising from counterfactual
reasoning. Our research was done independently of these works.

\subsection{Causal graphs and notation}

Causal graphs are a convenient way of organizing assumptions about the data
generating process. We will generally consider causal graphs involving a
protected attribute~$A,$ a set of proxy variables~$P,$ features~$X,$ a
predictor~$R$ and sometimes an observed outcome~$Y.$ For background on causal
graphs see~\cite{Pearl2009}. In the present paper a \emph{causal graph} is a
directed, acyclic graph whose nodes represent random variables. A \emph{directed
path} is a sequence of distinct nodes~$V_1, \dots, V_k$, for~$k \ge 2$, such
that~$V_i \to V_{i+1}$ for all~$i \in \{1, \dots,k-1\}$. We say a directed path
is \emph{blocked by a set of nodes~$Z$}, where~$V_1, V_k \notin Z$, if~$V_i \in
Z$ for some~$i \in \{2, \dots, k-1 \}$.\footnote{As it is not needed in our
work, we do not discuss the graph-theoretic notion of d-separation.}

A \emph{structural equation model} is a set of equations~$V_i = f_i(pa(V_i),
N_i)$, for~$i \in \{1, \dots, n\}$, where~$pa(V_i)$ are the parents of~$V_i$,
\ie{} its \emph{direct causes}, and the~$N_i$ are independent noise variables. We interpret
these equations as assignments. Because we assume acyclicity, starting from the
roots of the graph, we can recursively compute the other variables, given the
noise variables. This leads us to view the structural equation model and its
corresponding graph as a \emph{data generating model}. The predictor~$R$ maps
inputs, e.g., the features~$X$, to a predicted output. Hence we model it as a
childless node, whose parents are its input variables. Finally, note that given
the noise variables, a structural equation model entails a unique joint
distribution; however, the same joint distribution can usually be entailed by
multiple structural equation models corresponding to distinct causal structures.

\section{Unresolved discrimination and limitations of observational criteria}
\label{sec:limitations}

\begin{wrapfigure}{r}{0.31\textwidth}
  \begin{center}
  \begin{tikzpicture}
    \pgfmathsetmacro{\d}{1}
    \node (A) at (0,0) {$A$};
    \node (X) at (-0.6,-1) {$X$};
    \node (R) at (0.6,-1) {$\hR$};
    \draw[->] (A) -- (X);
    \draw[->] (A) -- (R);
    \draw[->] (X) -- (R);
  \end{tikzpicture}%
  \caption{The admission decision~$\hR$ does not only directly depend on
  gender~$A$, but also on department choice~$X$, which in turn is also
  affected by gender~$A$.}
  \label{fig:pearl}
  \end{center}
\end{wrapfigure}
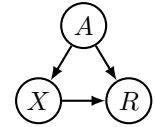

To bear out the limitations of observational criteria, we turn to Pearl's
commentary on claimed gender discrimination in Berkeley college
admissions~\cite[Section 4.5.3]{Pearl2009}. Bickel~\cite{Bickel1975} had shown
earlier that a lower college-wide admission rate for women than for men was
explained by the fact that women applied in more competitive departments. When
adjusted for department choice, women experienced a slightly higher acceptance
rate compared with men. From the causal point of view, what matters is the
\emph{direct effect} of the protected attribute (here, gender~$A$) on the
decision (here, college admission~$\hR$) that cannot be ascribed to a
\emph{resolving variable} such as department choice~$X$, see \figref{fig:pearl}.
We shall use the term \emph{resolving variable} for any variable in the causal
graph that is influenced by~$A$ in a manner that we accept as
non-discriminatory. With this convention, the criterion can be stated as
follows.

\begin{definition}[Unresolved discrimination]
  A variable~$V$ in a causal graph exhibits \emph{unresolved discrimination} if
  there exists a directed path from~$A$ to~$V$ that is not blocked by a
  resolving variable and~$V$ itself is non-resolving.
\end{definition}

Pearl's commentary is consistent with what we call the \emph{skeptic viewpoint}. All paths from
the protected attribute~$A$ to $R$ are problematic, unless they are justified by
a resolving variable. The presence of unresolved discrimination in the
predictor~$R$ is worrisome and demands further scrutiny. In practice,~$R$ is not
a priori part of a given graph. Instead it is our objective to construct it as a
function of the features~$X$, some of which might be resolving. Hence we should
first look for unresolved discrimination in the features. A canonical way to
avoid unresolved discrimination in~$R$ is to only input the set of features that
do not exhibit unresolved discrimination. However, the remaining features might
be affected by non-resolving \emph{and} resolving variables. In
\secref{sec:criteria} we investigate whether one can exclusively remove
unresolved discrimination from such features. A related notion of ``explanatory
features'' in a non-causal setting was introduced in~\cite{Kamiran2013}.

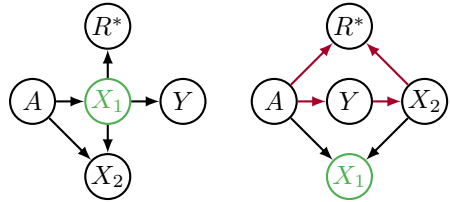
\begin{wrapfigure}{r}{0.5\linewidth}
  \centering
\begin{tikzpicture}
  \pgfmathsetmacro{\d}{1}
  \node (A) at (0,0) {$A$};
  \node[thick, Green] (X1) at (\d,0) {$X_1$};
  \node (Y) at (2*\d,0) {$Y$};
  \node (X2) at (\d,-\d) {$X_2$};
  \node (R) at (\d,\d) {$\hR^*$};
  \draw[->] (A) -- (X1);
  \draw[->] (X1) -- (Y);
  \draw[->] (A) -- (X2);
  \draw[->] (X1) -- (X2);
  \draw[->] (X1) -- (R);
\end{tikzpicture}%
  \hspace{0.5cm}
\begin{tikzpicture}
  \pgfmathsetmacro{\d}{1}
  \node (A) at (0,0) {$A$};
  \node (Y) at (\d,0) {$Y$};
  \node[thick, Green] (X1) at (\d,-\d) {$X_1$};
  \node (X2) at (2*\d,0) {$X_2$};
  \node (R) at (\d,\d) {$\hR^*$};
  \draw[->, DarkRed] (A) -- (Y);
  \draw[->, DarkRed] (Y) -- (X2);
  \draw[->, DarkRed] (X2) -- (R);
  \draw[->, DarkRed] (A) -- (R);
  \draw[->] (A) -- (X1);
  \draw[->] (X2) -- (X1);
\end{tikzpicture}%
  \caption{Two graphs that may generate the same joint distribution for the
  Bayes optimal unconstrained predictor~$\hR^*$. If~$X_1$ is a resolving
  variable,~$R^*$ exhibits unresolved discrimination in the right graph (along
  the red paths), but not in the left one.}
  \label{fig:unidentifiability}
\end{wrapfigure}

The definition of unresolved discrimination in a predictor has some interesting
special cases worth highlighting. If we take the set of resolving variables to
be empty, we intuitively get a causal analog of demographic parity. No directed paths from~$A$ to~$R$ are allowed, but~$A$ and~$R$ can
still be statistically dependent. Similarly, if we choose the set of resolving
variables to be the singleton set~$\{Y\}$ containing the true outcome, we obtain
a causal analog of equalized odds where strict independence is not necessary.
The causal intuition implied by ``the protected attribute should not affect the
prediction'', and ``the protected attribute can only affect the prediction when
the information comes through the true label'', is neglected by (conditional)
statistical independences~$A \indep R$, and~$A \indep R \given Y$, but well
captured by only considering dependences mitigated along directed causal paths.

We will next show that observational criteria are fundamentally unable to
determine whether a predictor exhibits unresolved discrimination or not. This
is true even if the predictor is \emph{Bayes optimal}. In passing, we also note that
fairness criteria such as equalized odds may or may not exhibit unresolved
discrimination, but this is again something an observational criterion cannot
determine.

\begin{theorem}\label{thm:unidentifiability}
  Given a joint distribution over the protected attribute~$A$, the true
  label~$Y$, and some features~$X_1, \dots, X_n$, in which we have already
  specified the resolving variables, no observational criterion can generally
  determine whether the Bayes optimal unconstrained predictor or the Bayes
  optimal equal odds predictor exhibit unresolved discrimination.
\end{theorem}

All proofs for the statements in this paper are in the supplementary material.

The two graphs in \figref{fig:unidentifiability} are taken
from~\cite{Hardt2016}, which we here reinterpret in the causal context to prove
Theorem~\ref{thm:unidentifiability}. We point out that there is an established
set of conditions under which unresolved discrimination can, in fact, be
determined from observational data. Note that the two graphs are not Markov
equivalent. Therefore, to obtain the same joint distribution we must violate a
condition called \emph{faithfulness}.\footnote{If we do assume the Markov
condition and faithfulness, then conditional independences determine the graph
up to its so called \emph{Markov equivalence class}.} We later argue that
violation of faithfulness is by no means pathological, but emerges naturally
when designing predictors.  In any case, interpreting conditional dependences
can be difficult in practice~\cite{Cornia2014}.

\section{Proxy discrimination and interventions}
\label{sec:proxies}

We now turn to an important aspect of our framework. Determining causal effects
in general requires modeling interventions. Interventions on deeply rooted
individual properties such as \emph{gender} or \emph{race} are notoriously
difficult to conceptualize---especially at an individual level, and impossible
to perform in a randomized trial. VanderWeele et al.~\cite{VanderWeele2014}
discuss the problem comprehensively in an epidemiological setting. From a
machine learning perspective, it thus makes sense to separate the protected
attribute~$A$ from its potential \emph{proxies}, such as name, visual features,
languages spoken at home, etc. Intervention based on proxy variables poses a
more manageable problem. By deciding on a suitable proxy we can find an adequate
mounting point for determining and removing its influence on the prediction.
Moreover, in practice we are often limited to imperfect measurements of~$A$ in
any case, making the distinction between root concept and proxy prudent.

As was the case with resolving variables, a \emph{proxy} is a priori nothing
more than a descendant of~$A$ in the causal graph that we choose to label as a
proxy. Nevertheless in reality we envision the proxy to be a clearly defined
observable quantity that is significantly correlated with~$A,$ yet in our view
should not affect the prediction.

\begin{definition}[Potential proxy discrimination]\label{def:proxy}
  A variable~$V$ in a causal graph exhibits \emph{potential proxy
  discrimination}, if there exists a directed path from~$A$ to~$V$ that is
  blocked by a proxy variable and~$V$ itself is not a proxy.
\end{definition}

Potential proxy discrimination articulates a causal criterion that is in a sense
dual to unresolved discrimination. From the \emph{benevolent viewpoint}, we \emph{allow} any path from~$A$
to~$\hR$ unless it passes through a proxy variable, which we consider worrisome.
This viewpoint acknowledges the fact that the influence of~$A$ on the graph may
be complex and it can be too restraining to rule out all but a few designated
features. In practice, as with unresolved discrimination, we can naively build
an unconstrained predictor based only on those features that do not exhibit
potential proxy discrimination. Then we must not provide~$P$ as input to~$R;$
unawareness, \ie{} excluding~$P$ from the inputs of~$R$, suffices. However, by
granting~$R$ access to~$P$, we can carefully tune the function~$R(P, X)$ to
cancel the implicit influence of~$P$ on features~$X$ that exhibit potential
proxy discrimination by the explicit dependence on $P$. Due to this possible
cancellation of paths, we called the path based criterion \emph{potential} proxy
discrimination. When building predictors that exhibit no \emph{overall proxy
discrimination}, we precisely aim for such a cancellation.

Fortunately, this idea can be conveniently expressed by an \emph{intervention}
on~$P$, which is denoted by~$do(P=p)$~\cite{Pearl2009}. Visually, intervening
on~$P$ amounts to removing all incoming arrows of~$P$ in the graph;
algebraically, it consists of replacing the structural equation of $P$ by $P=p$,
\ie{} we put point mass on the value~$p$.

\begin{definition}[Proxy discrimination]\label{def:overalldisc}
  A predictor~$\hR$ exhibits no \emph{proxy discrimination} based on a proxy~$P$
  if for all~$p, p'$
  \begin{equation}\label{eq:proxy-eq}
    \bP(\hR \given do(P=p) ) = \bP(\hR \given do(P=p') ) \eqp
  \end{equation}
\end{definition}

The interventional characterization of proxy discrimination leads to a simple
procedure to remove it in causal graphs that we will turn to in the next
section. It also leads to several natural variants of the definition that we
discuss in Section~\ref{sec:variants}. We remark that
Equation~\eqref{eq:proxy-eq} is an equality of probabilities in the
``do-calculus'' that cannot in general be inferred by an observational method,
because it depends on an underlying causal graph, see the discussion
in~\cite{Pearl2009}. However, in some cases, we do not need to resort to
interventions to avoid proxy discrimination.

\begin{proposition}\label{pro:unawareness}
  If there is no directed path from a proxy to a feature, unawareness avoids
  proxy discrimination.
\end{proposition}

\section{Procedures for avoiding discrimination}
\label{sec:criteria}

Having motivated the two types of discrimination that we distinguish, we now
turn to building predictors that avoid them in a given causal model. First, we
remark that a more comprehensive treatment requires individual judgement of not
only variables, but the legitimacy of every existing path that ends in~$R$,
\ie{} evaluation of \emph{path-specific effects}~\cite{Zhang2017b, Nabi2017},
which is tedious in practice. The natural concept of proxies and resolving
variables covers most relevant scenarios and allows for natural removal
procedures.

\begin{figure}
\begin{center}
  \begin{minipage}{0.5\textwidth}
    \begin{tikzpicture}
      \pgfmathsetmacro{\dy}{0.9}
      \node (NP) at (-1,0) {$N_P$};
      \node[thick, Green] (A) at (0,0) {$A$};
      \node (NX) at (1,0) {$N_X$};
      \node[thick, DarkRed] (P) at (-0.5,-\dy) {$P$};
      \node (X) at (0.5,-\dy) {$X$};
      \node (R) at (0,-2*\dy) {$\hR$};
      \node[draw=none] at (-0.7,-2*\dy) {$\tcG$};
      \draw[->] (NP) -- (P);
      \draw[->] (A) -- (P);
      \draw[->] (A) -- (X);
      \draw[->] (NX) -- (X);
      \draw[->] (P) -- (X);
      \draw[->] (P) -- (R);
      \draw[->] (X) -- (R);
    \end{tikzpicture}%
    \hspace{0.5cm}
    \begin{tikzpicture}
      \pgfmathsetmacro{\dy}{0.9}
      \node (NP) at (-1,0) {$N_P$};
      \node[thick, Green] (A) at (0,0) {$A$};
      \node (NX) at (1,0) {$N_X$};
      \node[thick, DarkRed, double] (P) at (-0.5,-\dy) {$P$};
      \node (X) at (0.5,-\dy) {$X$};
      \node (R) at (0,-2*\dy) {$\hR$};
      \node[draw=none] at (-0.7,-2*\dy) {$\cG$};
      \draw[thick, DarkRed, ->] (P) -- (X);
      \draw[thick, DarkRed, ->] (X) -- (R);
      \draw[thick, DarkRed, ->] (P) -- (R);
      \draw[->] (A) -- (X);
      \draw[->] (NX) -- (X);
    \end{tikzpicture}%
    \caption{A template graph~$\tcG$ for proxy discrimination (left) with its
    intervened version~$\cG$ (right). While from the benevolent viewpoint we do not generically prohibit any
    influence from~$A$ on~$\hR$, we want to guarantee that the proxy~$P$ has no
    overall influence on the prediction, by adjusting~$P \to \hR$ to cancel the
    influence along~$P \to X \to \hR$ in the intervened graph.}
  \label{fig:proxytemplate}
  \end{minipage}%
  \hfill
  \begin{minipage}{0.45\textwidth}
    \begin{tikzpicture}
      \pgfmathsetmacro{\dy}{0.9}
      \node (NE) at (-1,0) {$N_E$};
      \node[thick, DarkRed] (A) at (0,0) {$A$};
      \node (NX) at (1,0) {$N_X$};
      \node[thick, Green] (E) at (-0.5,-\dy) {$E$};
      \node (X) at (0.5,-\dy) {$X$};
      \node (R) at (0,-2*\dy) {$\hR$};
      \node[draw=none] at (-0.7,-2*\dy) {$\tcG$};
      \draw[->] (NE) -- (E);
      \draw[->] (A) -- (E);
      \draw[->] (A) -- (X);
      \draw[->] (NX) -- (X);
      \draw[->] (E) -- (X);
      \draw[->] (E) -- (R);
      \draw[->] (X) -- (R);
    \end{tikzpicture}%
    \hspace{0.5cm}
    \begin{tikzpicture}
      \pgfmathsetmacro{\dy}{0.9}
      \node (NE) at (-1,0) {$N_E$};
      \node[thick, DarkRed] (A) at (0,0) {$A$};
      \node (NX) at (1,0) {$N_X$};
      \node[thick, Green, double] (E) at (-0.5,-\dy) {$E$};
      \node (X) at (0.5,-\dy) {$X$};
      \node (R) at (0,-2*\dy) {$\hR$};
      \node[draw=none] at (-0.7,-2*\dy) {$\cG$};
      \draw[thick, DarkRed, ->] (A) -- (X);
      \draw[thick, DarkRed, ->] (X) -- (R);
      \draw[->] (NX) -- (X);
      \draw[->] (E) -- (X);
      \draw[->] (E) -- (R);
    \end{tikzpicture}%
    \caption{A template graph~$\tcG$ for unresolved discrimination (left) with
    its intervened version~$\cG$ (right). While from the skeptical viewpoint we generically do not want~$A$
    to influence~$\hR$, we first intervene on~$E$ interrupting all paths
    through~$E$ and only cancel the remaining influence on~$A$ to~$R$.}
    \label{fig:explanatorytemplate}
  \end{minipage}
\end{center}
\end{figure}

\subsection{Avoiding proxy discrimination}
\label{subsec:proxydiscrimination}

While presenting the general procedure, we illustrate each step in the example
shown in \figref{fig:proxytemplate}. A protected attribute~$A$ affects a
proxy~$P$ as well as a feature~$X$. Both~$P$ and~$X$ have additional unobserved
causes~$N_P$ and~$N_X$, where~$N_P, N_X, A$ are pairwise independent. Finally,
the proxy also has an effect on the features~$X$ and the predictor~$\hR$ is
a function of~$P$ and~$X$. Given labeled training data, our task is to find a
good predictor that exhibits no proxy discrimination within a hypothesis class
of functions~$R_{\theta}(P, X)$ parameterized by a real valued vector~$\theta$.

We now work out a formal procedure to solve this task under specific assumptions
and simultaneously illustrate it in a fully linear example, \ie{} the structural
equations are given by
\begin{equation*}
  P = \alpha_P A + N_P, \qquad
  X = \alpha_X A + \beta P + N_X, \qquad
  \hR_{\theta} = \lambda_P P + \lambda_X X \eqp
\end{equation*}
Note that we choose linear functions parameterized by~$\theta = (\lambda_P,
\lambda_X)$ as the hypothesis class for~$R_{\theta}(P, X)$.

We will refer to the \emph{terminal ancestors of a node~$V$ in a causal
graph~$\mathcal{D}$}, denoted by~$ta^{\mathcal{D}}(V)$, which are those
ancestors of~$V$ that are also root nodes of~$\mathcal{D}$. Moreover, in the
procedure we clarify the notion of \emph{expressibility}, which is an assumption
about the relation of the given structural equations and the hypothesis class we
choose for~$R_{\theta}$.

\begin{proposition}\label{pro:proxydisc}
  If there is a choice of parameters~$\theta_0$ such that~$R_{\theta_0}(P,X)$ is
  constant with respect to its first argument and the structural equations are
  \emph{expressible}, the following procedure returns a predictor from the given
  hypothesis class that exhibits no proxy discrimination and is non-trivial in
  the sense that it can make use of features that exhibit potential proxy
  discrimination.
\end{proposition}

\begin{enumerate}
  \item Intervene on~$P$ by removing all incoming arrows and replacing the
    structural equation for~$P$ by~$P=p$. For the example in
    \figref{fig:proxytemplate},
    \begin{equation} \label{eq:proxyR}
      P = p, \qquad
      X = \alpha_X A + \beta P + N_X, \qquad
      \hR_{\theta} = \lambda_P P + \lambda_X X \eqp
    \end{equation}
  \item Iteratively substitute variables in the equation for~$\hR_{\theta}$ from
    their structural equations until only root nodes of the intervened graph are
    left, \ie{} write~$R_{\theta}(P, X)$ as~$R_{\theta}(P, g(ta^{\cG}(X)))$ for
    some function~$g$. In the example,~$ta(X) = \{A, P, N_X\}$ and
    \begin{equation}\label{eq:proxyRroots}
      \hR_{\theta} = (\lambda_P + \lambda_X \beta) p + \lambda_X (\alpha_X A + N_X) \eqp
    \end{equation}
  \item We now require the distribution of~$\hR_{\theta}$
    in~\eqref{eq:proxyRroots} to be independent of~$p$, \ie{} for all~$p, p'$
    \begin{equation}\label{eq:proxyconstraint}
      \bP((\lambda_P + \lambda_X \beta) p  + \lambda_X (\alpha_X A + N_X)) =
      \bP((\lambda_P + \lambda_X \beta) p' + \lambda_X (\alpha_X A + N_X)) \eqp
    \end{equation}
    We seek to write the predictor as a function of~$P$ and all the other roots
    of~$\cG$ separately. If our hypothesis class is such that there
    exists~$\tilde{\theta}$ such that~$R_{\theta}(P, g(ta(X))) =
    R_{\tilde{\theta}}(P, \tilde{g}(ta(X) \setminus \{P\}))$, we call the
    structural equation model and hypothesis class specified
    in~\eqref{eq:proxyR} \emph{expressible}. In our example, this is possible
    with~$\tilde{\theta} = (\lambda_P + \lambda_X \beta, \lambda_X)$
    and~$\tilde{g} = \alpha_X A + N_X$. Equation~\eqref{eq:proxyconstraint} then
    yields the \emph{non-discrimination constraint}~$\tilde{\theta} = \theta_0$.
    Here, a possible~$\theta_0$ is~$\theta_0 = (0, \lambda_X)$, which simply
    yields~$\lambda_P = -\lambda_X \beta$.
  \item Given labeled training data, we can optimize the
    predictor~$\hR_{\theta}$ within the hypothesis class as given in
    \eqref{eq:proxyR}, subject to the non-discrimination constraint. In the
    example
    \begin{equation*}
      \hR_{\theta} = -\lambda_X \beta P + \lambda_X X = \lambda_X (X - \beta P) \eqc
    \end{equation*}
    with the free parameter~$\lambda_X \in \bR$.
\end{enumerate}

In general, the non-discrimination constraint~\eqref{eq:proxyconstraint} is by
construction just~$\bP(R \given do(P=p)) = \bP(R \given do(P=p'))$, coinciding
with Definition~\ref{def:overalldisc}.  Thus Proposition~\ref{pro:proxydisc}
holds by construction of the procedure. The choice of~$\theta_0$ strongly
influences the non-discrimination constraint.  However, as the example shows, it
allows~$R_{\theta}$ to exploit features that exhibit potential proxy
discrimination.

\subsection{Avoiding unresolved discrimination}
\label{subsec:unexplaineddiscrimination}

We proceed analogously to the previous subsection using the example graph in
\figref{fig:explanatorytemplate}. Instead of the proxy, we consider a resolving
variable~$E$. The causal dependences are equivalent to the ones in
\figref{fig:proxytemplate} and we again assume linear structural equations
\begin{equation*}
  E = \alpha_E A + N_E, \qquad
  X = \alpha_X A + \beta E + N_X, \qquad
  \hR_{\theta} = \lambda_E E + \lambda_X X \eqp
\end{equation*}

Let us now try to adjust the previous procedure to the context of avoiding
unresolved discrimination.

\begin{enumerate}
  \item Intervene on~$E$ by fixing it to a random variable~$\eta$
    with~$\bP(\eta) = \bP(E)$, the marginal distribution of~$E$ in~$\tcG$, see
    \figref{fig:explanatorytemplate}. In the example we find
    \begin{align} \label{eq:explanatoryR}
      E = \eta, \qquad
      X = \alpha_X A + \beta E + N_X, \qquad
      \hR_{\theta} = \lambda_E E + \lambda_X X \eqp
    \end{align}
  \item By iterative substitution write $R_{\theta}(E, X)$ as~$R_{\theta}(E,
    g(ta^{\cG}(X)))$ for some function~$g$, \ie{} in the example
    \begin{equation}\label{eq:explanatoryRroots}
      \hR_{\theta} = (\lambda_E + \lambda_X \beta) \eta + \lambda_X \alpha_X A + \lambda_X N_X \eqp
    \end{equation}
  \item We now demand the distribution of~$\hR_{\theta}$
    in~\eqref{eq:explanatoryRroots} be invariant under interventions on~$A$,
    which coincides with conditioning on~$A$ whenever~$A$ is a root of~$\tcG$.
    Hence, in the example, for all~$a, a'$
    \begin{equation}\label{eq:explanatoryconstraint}
      \bP((\lambda_E + \lambda_X \beta) \eta + \lambda_X \alpha_X a + \lambda_X N_X)) =
      \bP((\lambda_E + \lambda_X \beta) \eta + \lambda_X \alpha_X a' + \lambda_X N_X)) \eqp
    \end{equation}
\end{enumerate}

Here, the subtle asymmetry between proxy discrimination and unresolved
discrimination becomes apparent. Because~$R_{\theta}$ is not explicitly a
function of~$A$, we cannot cancel implicit influences of~$A$ through~$X$. There
might still be a~$\theta_0$ such that~$R_{\theta_0}$ indeed
fulfils~\eqref{eq:explanatoryconstraint}, but there is no principled way for us
to construct it. In the example,~\eqref{eq:explanatoryconstraint} suggests the
obvious \emph{non-discrimination constraint}~$\lambda_X = 0$. We can then
proceed as before and, given labeled training data, optimize~$\hR_{\theta} =
\lambda_E E$ by varying~$\lambda_E$. However, by setting~$\lambda_X = 0$, we
also cancel the path~$A \to E \to X \to R$, even though it is blocked by a
resolving variable.  In general, if~$R_{\theta}$ does not have access to~$A$, we
can not adjust for unresolved discrimination without also removing resolved
influences from~$A$ on~$R_{\theta}$.

If, however,~$R_{\theta}$ is a function of~$A$, \ie{} we add the term~$\lambda_A
A$ to~$R_{\theta}$ in~\eqref{eq:explanatoryR}, the non-discrimination constraint
is~$\lambda_A = - \lambda_X \alpha_X$ and we can proceed analogously to the
procedure for proxies.

\subsection{Relating proxy discriminations to other notions of fairness}
\label{sec:variants}

Motivated by the algorithm to avoid proxy discrimination, we discuss some
natural variants of the notion in this section that connect our interventional
approach to individual fairness and other proposed criteria. We consider a
generic graph structure as shown on the left in \figref{fig:generalgraph}. The
proxy~$P$ and the features~$X$ could be multidimensional. The empty circle in
the middle represents any number of variables forming a DAG that respects the
drawn arrows.  \Figref{fig:proxytemplate} is an example thereof. All dashed
arrows are optional depending on the specifics of the situation.

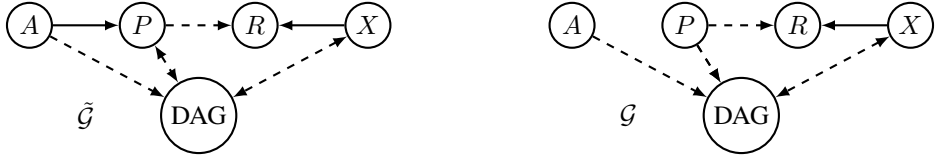
\begin{figure}
  \centering
  \begin{tikzpicture}
  \pgfmathsetmacro{\d}{1.5}
  \pgfmathsetmacro{\dy}{1.2}
    \node[minimum height=1cm] (C) at (0,0) {DAG};
    \node (A) at (-1.5*\d, \dy) {$A$};
    \node (P) at (-0.5*\d, \dy) {$P$};
    \node (R) at ( 0.5*\d, \dy) {$\hR$};
    \node (X) at ( 1.5*\d, \dy) {$X$};
    \draw[->] (A) -- (P);
    \draw[->, dashed] (A) -- (C);
    \draw[<->, dashed] (P) -- (C);
    \draw[<->, dashed] (X) -- (C);
    \draw[->, dashed] (P) -- (R);
    \draw[->] (X) -- (R);
    \node[draw=none] (L) at (-\d,0) {$\tcG$};
  \end{tikzpicture}%
  \hspace{2cm}
  \begin{tikzpicture}
  \pgfmathsetmacro{\d}{1.5}
  \pgfmathsetmacro{\dy}{1.2}
    \node[minimum height=1cm] (C) at (0,0) {DAG};
    \node (A) at (-1.5*\d, \dy) {$A$};
    \node (P) at (-0.5*\d, \dy) {$P$};
    \node (R) at ( 0.5*\d, \dy) {$\hR$};
    \node (X) at ( 1.5*\d, \dy) {$X$};
    \draw[->, dashed] (A) -- (C);
    \draw[->, dashed] (P) -- (C);
    \draw[->, dashed] (P) -- (C);
    \draw[<->, dashed] (X) -- (C);
    \draw[->, dashed] (P) -- (R);
    \draw[->] (X) -- (R);
    \node[draw=none] (L) at (-\d,0) {$\cG$};
  \end{tikzpicture}%
  \caption{{\em Left:} A generic graph~$\tcG$ to describe proxy discrimination.
  {\em Right:} The graph corresponding to an intervention on~$P$. The circle
  labeled ``DAG'' represents any sub-DAG of~$\tcG$ and~$\cG$ containing an
  arbitrary number of variables that is compatible with the shown arrows. Dashed
  arrows can, but do not have to be present in a given scenario.}
\label{fig:generalgraph}
\end{figure}

\begin{definition}\label{def:nondisc}
  A predictor~$\hR$ exhibits no \emph{individual proxy discrimination}, if for
  all~$x$ and all~$p, p'$
  \begin{equation*}
    \bP(\hR \given do(P=p), X=x) = \bP(\hR \given do(P=p'), X=x) \eqp
  \end{equation*}
  A predictor~$\hR$ exhibits no \emph{proxy discrimination in expectation}, if
  for all~$p, p'$
  \begin{equation*}
    \E[\hR \given do(P=p)] = \E[\hR \given do(P=p')] \eqp \label{eq:inexp}
  \end{equation*}
\end{definition}

Individual proxy discrimination aims at comparing examples with the same
features~$X$, for different values of~$P$. Note that this can be individuals
with different values for the unobserved non-feature variables. A true
individual-level comparison of the form ``What would have happened to me, if I
had always belonged to another group'' is captured by \emph{counterfactuals} and
discussed in~\cite{Kusner2017, Nabi2017}.

For an analysis of proxy discrimination, we need the structural equations
for~$P, X, \hR$ in \figref{fig:generalgraph}
\begin{subequations}
  \begin{align*}
    P &= \hat{f}_P (pa(P)) \eqc \\
    X &= \hat{f}_X (pa(X)) = f_X (P, ta^{\cG} (X) \setminus \{P\}) \eqc \\
    \hR &= \hat{f}_{\hR} (P, X) = f_{\hR} (P, ta^{\cG} (\hR) \setminus \{P\}) \eqp
  \end{align*}
\end{subequations}
For convenience, we will use the notation~$ta^{\cG}_{P}(X) := ta^{\cG} (X)
\setminus \{P\}$. We can find~$f_X, f_R$ from~$\hat{f}_X, \hat{f}_R$ by first
rewriting the functions in terms of root nodes of the \emph{intervened graph},
shown on the right side of \figref{fig:generalgraph}, and then assigning the
\emph{overall} dependence on $P$ to the first argument.

We now compare proxy discrimination to other existing notions.

\begin{theorem}\label{thm:rvtodistr}
  Let the influence of~$P$ on~$X$ be additive and linear, \ie{}
  \begin{equation*}
    X = f_X (P, ta^{\cG}_P(X)) = g_X(ta^{\cG}_P(X)) + \mu_X P
  \end{equation*}
  for some function~$g_X$ and~$\mu_X \in \bR$. Then any predictor of the form
  \begin{equation*}
    \hR = r(X - \E [X \given do(P)])
  \end{equation*}
  for some function~$r$ exhibits no proxy discrimination.
\end{theorem}

Note that in general~$\E[X \given do(P)] \neq \E[X \given P]$. Since in practice
we only have observational data from~$\tcG$, one cannot simply build a
predictor based on the ``regressed out features''~$\tilde{X} := X - \E[X \given
P]$ to avoid proxy discrimination. In the scenario of
\figref{fig:proxytemplate}, the direct effect of~$P$ on~$X$ along the arrow~$P
\to X$ in the left graph cannot be estimated by~$\E[X \given P]$, because of
the common confounder~$A$. The desired interventional expectation~$\E[X \given
do(P)]$ coincides with~$\E[X \given P]$ only if one of the arrows~$A \to P$
or~$A \to X$ is not present. Estimating direct causal effects is a hard problem,
well studied by the causality community and often involves instrumental
variables~\cite{Angrist2001}.

This cautions against the natural idea of using~$\tilde{X}$ as a ``fair
representation'' of~$X$, as it implicitly neglects that we often want to remove
the effect of proxies and not the protected attribute. Nevertheless, the notion
agrees with our interventional proxy discrimination in some cases.

\begin{corollary}\label{cor:rvtodistr}
  Under the assumptions of Theorem~\ref{thm:rvtodistr}, if all directed paths
  from any ancestor of~$P$ to~$X$ in the graph~$\cG$ are blocked by~$P$, then
  any predictor based on the \emph{adjusted features}~$\tilde{X} := X - \E[X
  \given P]$ exhibits no proxy discrimination and can be learned from the
  observational distribution~$\bP(P, X, Y)$ when target labels~$Y$ are
  available.
\end{corollary}

Our definition of proxy discrimination in expectation (\ref{eq:inexp}) is
motivated by a weaker notion proposed in~\cite{Calders2010}. It asks for the
expected outcome to be the same across the different populations~$\E[R \given
P=p] = \E[R \given P=p'].$ Again, when talking about proxies, we must be careful
to distinguish conditional and interventional expectations, which is captured by
the following proposition and its corollary.

\begin{proposition}\label{pro:inexpectation}
  Any predictor of the form~$\hR = \lambda (X - \E[X \given do(P)]) + c$
  for~$\lambda, c \in \bR$ exhibits no proxy discrimination in expectation.
\end{proposition}

From this and the proof of Corollary~\ref{cor:rvtodistr} we conclude the
following Corollary.

\begin{corollary}
  If all directed paths from any ancestor of~$P$ to~$X$ are blocked by~$P$, any
  predictor of the form~$\hR = r(X - \E[X \given P])$ for linear~$r$ exhibits no
  proxy discrimination in expectation and can be learned from the observational
  distribution~$\bP(P, X, Y)$ when target labels~$Y$ are available.
\end{corollary}

\section{Conclusion}
\label{sec:conclusion}

The goal of our work is to assay fairness in machine learning within the context
of causal reasoning. This perspective naturally addresses shortcomings of
earlier statistical approaches. Causal fairness criteria are suitable whenever
we are willing to make assumptions about the (causal) generating process
governing the data.  Whilst not always feasible, the causal approach naturally
creates an incentive to scrutinize the data more closely and work out plausible
assumptions to be discussed alongside any conclusions regarding fairness.

Key concepts of our conceptual framework are \emph{resolving variables} and
\emph{proxy variables} that play a dual role in defining causal discrimination
criteria. We develop a practical procedure to remove proxy discrimination given
the structural equation model and analyze a similar approach for unresolved
discrimination. In the case of proxy discrimination for linear structural
equations, the procedure has an intuitive form that is similar to heuristics
already used in the regression literature. Our framework is limited by the
assumption that we can construct a valid causal graph. The removal of proxy
discrimination moreover depends on the functional form of the causal
dependencies. We have focused on the conceptual and theoretical analysis, and
experimental validations are beyond the scope of the present work.

The causal perspective suggests a number of interesting new directions at the
technical, empirical, and conceptual level. We hope that the framework and
language put forward in our work will be a stepping stone for future
investigations.

\clearpage{}
\printbibliography[title={References}]

\clearpage{}
\section*{Supplementary material}
\label{sec:appendix}

\subsection*{Proof of Theorem~\ref{thm:unidentifiability}}

\begin{theorem*}
  Given a joint distribution over the protected attribute~$A$, the true
  label~$Y$, and some features~$X_1, \dots, X_n$, in which we have already
  specified the resolving variables, no observational criterion can generally
  determine whether the Bayes optimal unconstrained predictor or the Bayes
  optimal equal odds predictor exhibit unresolved discrimination.
\end{theorem*}

\begin{proof}
  Let us consider the two graphs in \figref{fig:unidentifiability}. First, we
  show that these graphs can generate the same joint distribution~$\bP(A, Y,
  X_1, X_2, \hR^*)$ for the Bayes optimal unconstrained predictor~$\hR^*$.

  We choose the following structural equations for the graph on the
  left\footnote{$\sigma(x) = 1 / (1 + e^{-x})$}
  \begin{itemize}
    \item $A = Ber(\nicefrac{1}{2})$
    \item $X_1$ is a mixture of Gaussians $\cN(A+1, 1)$ with weight~$\sigma(2
      A)$ and $\cN(A-1, 1)$ with weight~$\sigma(-2 A)$
    \item $Y = Ber( \sigma(2 X_1) )$
    \item $X_2 = X_1 - A$
    \item $\hR^* =  X_1$
    \item ($\tilde{\hR} = X_2$)
  \end{itemize}
  where the Bernoulli distribution~$Ber(p)$ without a superscript has
  support~$\{-1, 1\}$.

  For the graph on the right, we define the structural equations
  \begin{itemize}
    \item $A = Ber(\nicefrac{1}{2})$
    \item $Y = Ber(\sigma(2 A))$
    \item $X_2 = \cN(Y,1)$
    \item $X_1 = A + X_2$
    \item $\hR^* = X_1$
    \item ($\tilde{\hR} = X_2$)
  \end{itemize}

  First we show that in both scenarios~$\hR^*$ is actually an optimal score.
  In the first scenario~$Y \indep A \given X_1$ and~$Y \indep X_2 \given X_1$
  thus the optimal predictor is only based on~$X_1$. We find
  \begin{equation}\label{eq:YgivenX1}
    \Pr(Y = y \given X_1 = x_1) = \sigma(2 x_1 y) \eqc
  \end{equation}
  which is monotonic in~$x_1$. Hence optimal classification is obtained by
  thresholding a score based only on~$\hR^* = X_1$.

  In the second scenario, because~$Y \indep X_1 \given \{A, X2\}$ the optimal
  predictor only depends on~$A, X_2$. We compute for the densities
  \begin{subequations}\label{eq:YgivenX2A}
    \begin{align}
      \bP(Y \given X_2, A) & = \frac{\bP(Y, X_2, A)}{\bP(X_2, A)} \\
      &= \frac{\bP(X_2, A \given Y) \bP(Y)}{\bP(X_2, A)} \\
      &= \frac{\bP(X_2 \given Y) \bP(A \given Y)
        \bP(Y)}{\bP(X_2, A)} \\
      &= \frac{\bP(X_2 \given Y) \frac{\bP(Y \given A)
        \bP(A)}{\bP(Y)} \bP(Y)}{\bP(X_2, A)} \\
      &= \frac{\bP(X_2 \given Y) \bP(Y \given A) \bP(A)}
        {\bP(X_2, A)} \eqc
    \end{align}
  \end{subequations}
  where for the third equal sign we use~$A \indep X_2 \given Y$. In the
  numerator we have
  \begin{equation}\label{eq:density}
    \bP(X_2 \given Y=y)(x_2) \bP(Y \given A=a)(y) \bP(A)(a)
    = f_{\cN(y,1)}(x_2) f_{Ber(\sigma(2 a))}(y) f_{Ber(\nicefrac{1}{2})}(a)\eqc
  \end{equation}
  where~$f_{D}$ is the probability density function of the distribution~$D$.
  The denominator can be computed by summing up~\eqref{eq:density} for~$y\in
  \{-1,1\}$. Overall this results in
  \begin{equation*}
    \Pr(Y = y \given X_2 = x_2, A=a) = \sigma(2 y (a + x_2)) \eqp
  \end{equation*}

  Since by construction~$X_1 = A + X_2$, the optimal predictor is again~$\hR^* =
  X_1$. If the joint distribution~$\bP(A, Y, \hR^*)$ is identical in the two
  scenarios, so are the joint distributions~$\bP(A, Y, X_1, X_2, \hR^*)$,
  because of~$X_1 = \hR^*$ and~$X_2 = X_1 - A$.

  To show that the joint distributions~$\bP(A, Y, \hR^*) = \bP(Y \given A,
  \hR^*) \bP(\hR^* \given A) \bP(A)$ are the same, we compare the conditional
  distributions in the factorization.

  Let us start with~$\bP(Y \given A, \hR^*)$.  Since~$\hR^* = X_1$ and in the
  first graph~$Y \indep A \given X_1$, we already found the distribution
  in~\eqref{eq:YgivenX1}. In the right graph,~$\bP(Y \given \hR^*, A) = \bP(Y
  \given X_2 + A, A) = \bP(Y \given X_2, A)$ which we have found
  in~\eqref{eq:YgivenX2A} and coincides with the conditional in the left graph
  because of~$X_1 = A + X_2$.

  Now consider~$\hR^* \given A$. In the left graph we have~$\bP(\hR^* \given A)
  = \bP(X_1 \given A)$ and the distribution~$\bP(X_1 \given A)$ is just the
  mixture of Gaussians defined in the structural equation model. In the right
  graph~$\hR^* = A + X_2 = Y + \cN(A,1)$ and thus~$\bP(\hR^* \given A) = \cN(A
  \pm 1)$ for~$Y = \pm 1$. Because of the definition of~$Y$ in the structural
  equations of the right graph, following a Bernoulli distribution with
  probability~$\sigma(2 A)$, this is the same mixture of Gaussians as the one we
  found for the left graph.

  Clearly the distribution of~$A$ is identical in both cases.

  Consequently the joint distributions agree.

  When~$X_1$ is an resolving variable, the optimal predictor in the left
  graph does not exhibit unresolved discrimination, whereas the graph on the
  right does.

  The proof for the equal odds predictor~$\tilde{\hR}$ is immediate once we
  show~$\tilde{\hR} = X_2$. This can be seen from the graph on the right,
  because here~$X_2 \indep A \given Y$ and both using~$A$ or~$X_1$ would
  violate the equal odds condition. Because the joint distribution in the left
  graph is the same,~$\tilde{\hR} = X_2$ is also the optimal equal odds
  score.\qedhere
\end{proof}

\subsection*{Proof of Proposition~\ref{pro:unawareness}}

\begin{proposition*}
  If there is no directed path from a proxy to a feature, unawareness avoids
  proxy discrimination.
\end{proposition*}

\begin{proof}
  An unaware predictor~${\hR}$ is given by~${\hR} = r (X)$ for some function~$r$
  and features~$X$. If there is no directed path from proxies~$P$ to~$X$,
  \ie{}~$P \notin ta^{\cG}(X)$, then~${\hR} = r(X) = r(ta^{\cG}(X)) =
  r(ta^{\cG}_P(X))$.  Thus~$\bP({\hR} \given do(P=p)) = \bP({\hR})$ for all~$p$,
  which avoids proxy discrimination.\qedhere
\end{proof}

\subsection*{Proof of Theorem~\ref{thm:rvtodistr}}

\begin{theorem*}
  Let the influence of~$P$ on~$X$ be additive and linear, \ie{}
  \begin{equation*}
    X = f_X (P, ta^{\cG}_P(X)) = g_X(ta^{\cG}_P(X)) + \mu_X P
  \end{equation*}
  for some function~$g_X$ and~$\mu_X \in \bR$. Then any predictor of the form
  \begin{equation*}
    \hR = r(X - \E [X \given do(P)])
  \end{equation*}
  for some function~$r$ exhibits no proxy discrimination.
\end{theorem*}

\begin{proof}
  It suffices to show that the argument of~$r$ is constant \wrt{} to~$P$,
  because then~$\hR$ and thus~$\bP(\hR)$ are invariant under changes of~$P$. We
  compute
  \begin{align*}
    \E[X \given do(P)] &= \E[ g_X(ta^{\cG}_P (X)) + \mu_X P \given do(P)] \\
    &= \usub{=0}{\E[ g_X(ta^{\cG}_P (X)) \given do(P)]} + \E[\mu_X P \given do(P)]\\
    &= \mu_X P \eqp
  \end{align*}
  Hence,
  \begin{equation*}
    X - \E[X \given do(P)] = g_X(ta^{\cG}_P (X))
  \end{equation*}
  is clearly constant \wrt{} to~$P$.\qedhere
\end{proof}

\subsection*{Proof of Corollary~\ref{cor:rvtodistr}}

\begin{corollary*}
  Under the assumptions of Theorem~\ref{thm:rvtodistr}, if all directed paths
  from any ancestor of~$P$ to~$X$ in the graph~$\cG$ are blocked by~$P$, then
  any predictor based on the \emph{adjusted features}~$\tilde{X} := X - \E[X
  \given P]$ exhibits no proxy discrimination and can be learned from the
  observational distribution~$\bP(P, X, Y)$ when target labels~$Y$ are
  available.
\end{corollary*}

\begin{proof}
  Let~$Z$ denote the set of ancestors of~$P$. Under the given assumptions~$Z
  \cap ta^{\cG} (X) = \emptyset$, because in~$\cG$ all arrows into~$P$ are
  removed, which breaks all directed paths from any variable in~$Z$ to~$X$ by
  assumption. Hence the distribution of~$X$ under an intervention on~$P$
  in~$\tcG$, where the influence of potential ancestors of~$P$ on~$X$ that
  does not go through~$P$ would not be affected, is the same as simply
  conditioning on~$P$. Therefore~$\E[X \given do(P)] = \E[X \given P]$, which
  can be computed from the joint observational distribution, since we
  observe~$X$ and~$P$ as generated by~$\tcG$.\qedhere
\end{proof}

\subsection*{Proof of Proposition~\ref{pro:inexpectation}}

\begin{proposition*}
  Any predictor of the form~$\hR = \lambda (X - \E[X \given do(P)]) + c$ for
  linear~$\lambda, c \in \bR$ exhibits no proxy discrimination in expectation.
\end{proposition*}

\begin{proof}
  We directly test the definition of proxy discrimination in expectation using
  the linearity of the expectation
  \begin{align*}
    \E[\hR \given do(P=p)] &= \E[\lambda (X - \E[X \given do(P)])
      + c \given do(P=p)] \\
    &= \lambda (\E[ X \given do(P=p)] - \E[X \given do(P=p)]) + c \\
    &= c \eqp
  \end{align*}
  This holds for any~$p$, hence proxy discrimination in expectation is
  achieved.\qedhere
\end{proof}

\subsection*{Additional statements}

Here we provide an additional statement that is a first step towards the
``opposite direction'' of Theorem~\ref{thm:rvtodistr}, \ie{} whether we can
infer information about the structural equations, when we are given a predictor
of a special form that does not exhibit proxy discrimination.

\begin{theorem*}
  Let the influence of~$P$ on~$X$ be additive and linear and let the influence
  of~$P$ on the argument of~$\hR$ be additive linear, \ie{}
  \begin{align*}
    f_X (ta^{\cG}(X)) &= g_X(ta^{\cG}_P(X)) + \mu_X P \\
    f_{\hR} (P, ta^{\cG}_P(X)) &= h(g_{\hR}(ta^{\cG}_P(X)) + \mu_{\hR} P)
  \end{align*}
  for some functions~$g_X, g_{\hR}$, real numbers~$\mu_X, \mu_{\hR}$ and a
  smooth, strictly monotonic function~$h$. Then any predictor that avoids proxy
  discrimination is of the form
  \begin{equation*}
    {\hR} = r(X - \E[X \given do(P)])
  \end{equation*}
  for some function~$r$.
\end{theorem*}

\begin{proof}
  From the linearity assumptions we conclude that
  \begin{equation*}
    \hat{f}_{\hR}(P, X) = h(g_X(ta^{\cG}_P(X)) + \mu_X P + \hat{\mu}_{\hR} P) \eqc
  \end{equation*}
  with~$\hat{\mu}_{\hR} = \mu_{\hR} - \mu_P$ and thus~$g_X = g_{\hR}$. That
  means that both the dependence of~$X$ on~$P$ along the path~$P \to \dots \to
  X$ as well as the direct dependence of~${\hR}$ on~$P$ along~$P \to {\hR}$
  are additive and linear.

  To avoid proxy discrimination, we need
  \begin{subequations}\label{eq:distreq}
  \begin{align}
    \bP ({\hR} \given do(P=p))
    &= \bP(h(g_{\hR}(ta^{\cG}_P(X)) + \mu_{\hR} p)) \\
    &\overset{!}{=} \bP(h(g_{\hR}(ta^{\cG}_P(X)) + \mu_{\hR} p'))
    = \bP ({\hR} \given do(P=p')) \eqp
  \end{align}
  \end{subequations}

  Because~$h$ is smooth an strictly monotonic, we can conclude that already the
  distributions of the argument of~$h$ must be equal, otherwise the
  transformation of random variables could not result in equal distributions,
  \ie{}
  \begin{equation*}
    \bP(g_{\hR}(ta^{\cG}_P(X)) + \mu_{\hR} p)
    \overset{!}{=} \bP(g_{\hR}(ta^{\cG}_P(X)) + \mu_{\hR} p') \eqp
  \end{equation*}
  Since, up to an additive constant, we are comparing the distributions of the
  \emph{same} random variable~$g_{\hR}(ta^{\cG}_P(X))$ and not merely
  identically distributed ones, the following condition is not only
  sufficient, but also necessary for~\eqref{eq:distreq}
  \begin{equation*}
    g_{\hR}(ta^{\cG}_P(X)) + \mu_{\hR} p
    \overset{!}{=} g_{\hR}(ta^{\cG}_P(X)) + \mu_{\hR} p' \eqp
  \end{equation*}
  This holds true for all~$p, p'$ only if~$\mu_{\hR} = 0$, which is equivalent
  to~$\hat{\mu}_{\hR} = - \mu_P$.

  Because as in the proof of~\ref{thm:rvtodistr}
  \begin{equation*}
    \E[X \given do(P)] = \mu_X P,
  \end{equation*}
  under the given assumptions any predictor that avoids proxy discrimination is
  simply
  \begin{equation*}
    {\hR} = X + \mu_{\hR} P = X - \E[X \given do(P)] \eqp
  \end{equation*}\qedhere
\end{proof}

\end{document}